\documentclass[sigconf]{acmart}
\renewcommand\footnotetextcopyrightpermission[1]{} 
\makeatletter
\renewcommand\@formatdoi[1]{\ignorespaces}
\makeatother

\usepackage{graphicx}
\usepackage{epstopdf}
\usepackage{algorithm}
\usepackage{algorithmic}
\newcommand{\X}{\mathbf{X}}
\newcommand{\R}{\mathbb{R}}
\newcommand{\xl}{\mathbf{x^l}}
\newcommand{\xr}{\mathbf{x^r}}
\newcommand{\SB}{\mathbf{S}}
\newcommand{\data}{{\em data faithful }}
\newcommand{\objective}{{\em objective faithful }}
\newcommand{\fig}[1]{Fig.(\ref{#1})}
\newcommand{\sect}[1]{Section \ref{#1}}
\newcommand{\eref}[1]{Eq.(\ref{#1})}
\newtheorem{theorem}{Theorem}

\AtBeginDocument{%
  \providecommand\BibTeX{{%
    \normalfont B\kern-0.5em{\scshape i\kern-0.25em b}\kern-0.8em\TeX}}}

\acmPrice{15.00}
\acmISBN{978-1-4503-XXXX-X/18/06}

\title{Simple is better: Making Decision Trees faster using random sampling}
\author{Vignesh Nanda Kumar}
\email{vignesh.nandakumar@aexp.com}
\affiliation{%
 \institution{AI Labs, American Express}
 \country{India}
}
\author{Narayanan U Edakunni}
\email{narayanan.u.edakunni@aexp.com}
\affiliation{%
  \institution{AI Labs, American Express}
  \country{India}
}

\begin{abstract}
In recent years, gradient boosted decision trees (GBDT) have become popular in building robust machine learning models on big data. The primary technique that has enabled these algorithms’ success has been distributing the computation while building the decision trees. A distributed decision tree building, in turn, has been enabled by building quantiles of the big datasets and choosing the candidate split points from these quantile sets. In XGBoost, for instance, a sophisticated quantile building algorithm is employed to identify the candidate split points for the decision trees. This method is often projected to yield better results when the computation is distributed. In this paper, we dispel the notion that these methods provide more accurate and scalable methods for building decision trees in a distributed manner. In a significant contribution, we show theoretically and empirically that choosing the split points uniformly at random provides the same or even better performance in terms of accuracy and computational efficiency. Hence, a simple random selection of points suffices for decision tree building compared to more sophisticated methods.
\end{abstract}

\begin{document}
\maketitle
\pagestyle{plain}

\section{Introduction}
XGBoost\cite{xgboost} algorithm's popularity has been mainly due to its ability to build effective GBDT models on large and complex data sets. Some of the techniques that allow it to build models efficiently are - distributed tree building, optimized data structures, and out-of-core computation. Out of these, the most important feature is building a decision tree in a distributed manner. The main bottleneck in building a decision tree in a distributed fashion is to determine the optimal value at which the node needs to be split. For the split finding step, a greedy algorithm iterates over all feature values, computing gain to find the best split point at a particular node. This is done for all nodes of the decision tree, making it inefficient. To alleviate this, there exist approximate split finding algorithms which speedup the split finding by enumerating only a limited set of data points for each feature. These algorithms summarize all feature values to a smaller set and then find the best split point from the summarized set. The summarization was previously applied to situations where decision trees were built on big datasets, or decision trees were built on streaming data\cite{spies,spdt}. These techniques for building decision trees have been adopted to more powerful GBDT methods like XGBoost\cite{xgboost}, CATBoost\cite{catboost}, and LightGBM\cite{Ke2017LightGBMAH}. These boosting algorithms build models using many decision trees; hence, any efficiency obtained in building decision trees is magnified in the overall ensemble model. Therefore, these GBDT algorithms deal with big data by efficiently building the individual decision trees in the ensemble model. For instance, in the original paper on XGBoost\cite{xgboost}, the authors come up with a weighted quantile building step that is reported to improve the accuracy of finding the split and at the same time constructing the decision trees efficiently. The paper claims this feature is a novel contribution of XGBoost compared to rest of its contemporaries. In XGBoost and algorithms similar to it, the data is summarized by sophisticated quantile building procedures, making sure that the data summary is faithful to the original data. In this paper, we show theoretically and empirically that {\em such methods of constructing sparse datasets using sophisticated quantile building procedures do not offer any more advantage in the accuracy of the ensembles built, compared to a simple random selection of points as candidate splits for a decision tree}. The results in this paper would help practitioners to build simple and scalable algorithms to build decision trees for big data.

The paper starts out examining previous work on efficiently building decision trees and the use of some of these techniques in GBDT methods, in Section \ref{sec:related}. We then provide a precise mathematical framework to explain the problem at hand and compare the efficacy of different quantization algorithm in Section \ref{sec:math}. This is followed by Section \ref{sec:rand} where we show that the sophisticated quantization algorithm is equivalent to a random selection of split candidates. Section \ref{sec:eval} provides an empirical comparison of XGBoost using a simple random sampling for data summarization with the original version of XGBoost with sophisticated quantile building steps.

\section{Related work}
\label{sec:related}
The decision tree is a useful ML technique that recursively partitions the input space and assigns specific class labels to different partitions. The partitions are constructed by first choosing the feature along which a partition needs to be created or equivalently a split needs to be created in the decision tree. This is then followed by choosing the value of the feature on which the split needs to be created. The split finding step needs to scan through the entire set of data points to determine the optimal feature and the appropriate value to split on. The split is chosen such that, by splitting on a particular node of the tree, the objective function defined over the data is maximised. Typical objective functions used in decision trees are GINI and entropy. The split point is chosen so that splitting at the particular value would give the maximum increase of the objective function compared to all other candidate splits. The split finding procedure is inefficient, especially when dealing with big data, since for every split, we need to scan over the entire data. There have been many algorithms that try to overcome this inefficiency by considering a subset of the data instead of the whole dataset. These algorithms fall into two different categories based on the way it constructs the subsets.
\begin{enumerate}
\item Data faithful: This class of algorithms typically try to build the subset of data such that the subset mimics the actual data as closely as possible. Hence we call these algorithms \data. These include algorithms like SS in CLOUDS\cite{clouds}, SPDT\cite{spdt}, XGBoost\cite{xgboost}, CATBoost\cite{catboost}, and LightGBM\cite{Ke2017LightGBMAH}. Here we have included boosting methods like XGBoost, CATBoost, and LightGBM along with decision tree building algorithms like CLOUDS and SPDT since the boosting algorithms build an ensemble of decision trees and use these approximate methods to build trees efficiently. Out of these, the SS variant of CLOUDS and XGBoost uses quantile approximation to build a subset over the training data. CATBoost has different possible ways of quantization provided as a choice to the user. LightGBM uses a histogram tree to construct the bins. In a quantile approximation, the idea is to construct bins over the data such that all of the training points fall into each bucket. Usually, separate bins are constructed for individual features of the data. Once the bins are constructed, a representative data point is chosen for each bin, serving as a split candidate.
\item Objective faithful: In this class of algorithms, the focus is to choose candidate splits so that they best represent the objective function optimized by the decision tree optimizer. However, it is challenging to represent the objective function because we do not have any knowledge about the objective. An algorithm like SSE of \cite{clouds} builds a subset of the training data in multiple rounds. The subset is constructed through a heuristic which tries to make the subset of data points faithful to the objective function and tries to include optimal data point in the subset. SPIES \cite{spies} is another such algorithm. The idea is to choose the split points such that there is a good chance that the best split point is included in the subset, hence ending up with a decision tree that is equivalent or close to the tree that would have been built if all of the data points are chosen.
\end{enumerate}
From the definitions, it is clear that the \data algorithms do not consider objective function while building the subset and entirely focus on approximating the input space. In this effort, these algorithms can end up choosing sub-optimal subsets as candidate splits. Representing the input space accurately does not guarantee that it would be able to capture splits that can potentially improve on the objective. This situation is illustrated in \fig{fig:mismatch}. On the other hand, \objective algorithms spend substantial effort choosing the approximating subset and hence are computationally inefficient. It is also not clear how to find the optimal splits over a completely unknown objective without actually evaluating the functions over the data points. In the following section, we quantify the expected error made by these algorithms in choosing the candidate splits while constructing the approximating subset and compare it with the simple approach of building the subset by sampling points uniformly at random.
\begin{figure}
\includegraphics[width=0.4\textwidth]{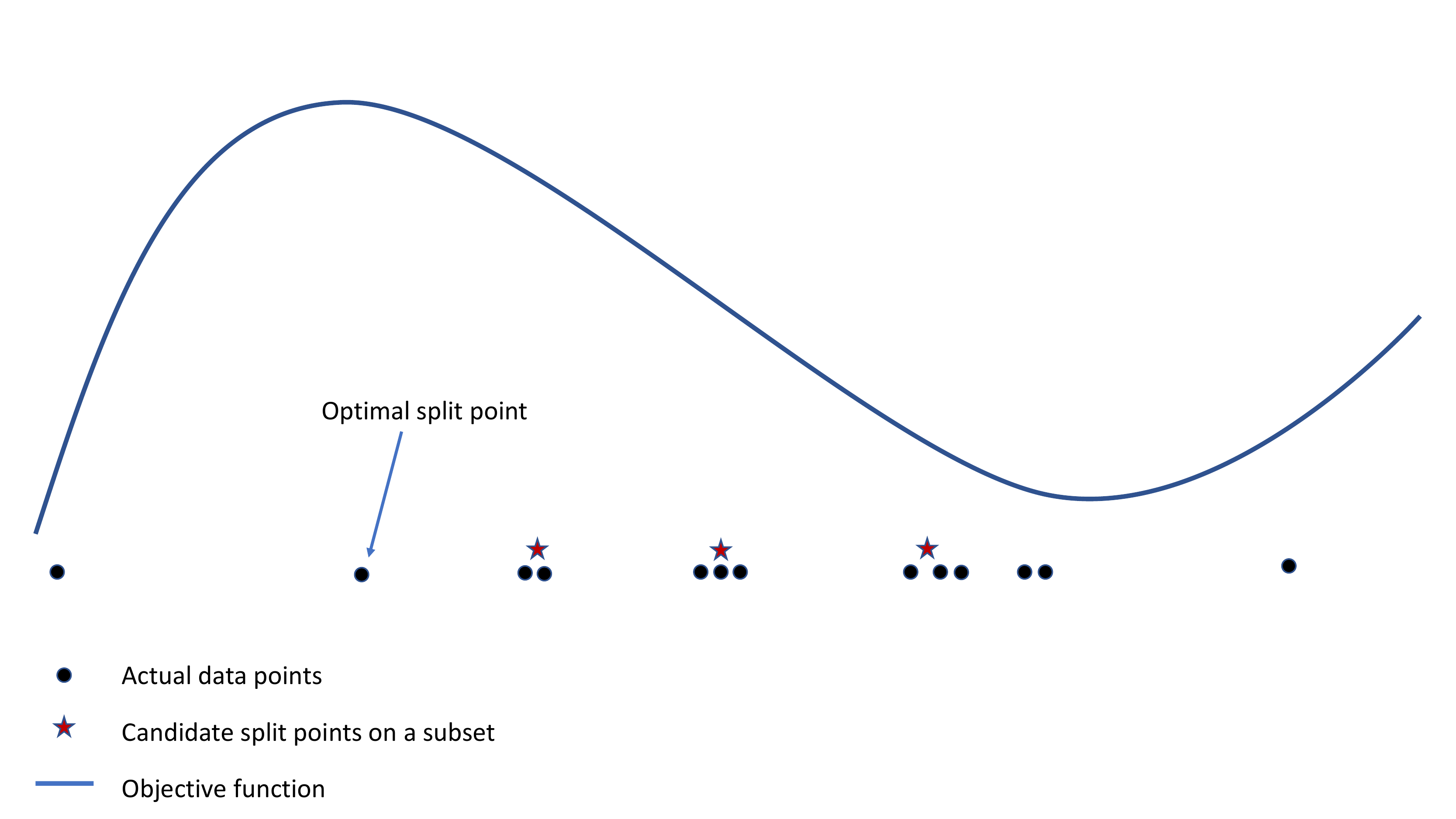}
\caption{Illustrates the situation where the \data algorithms try to be faithful to the data but are not able to include the optimal split point with respect to the objective. There are 13 data points which are grouped into 3 bins. The representatives of these bins are given by the red stars. As illustrated, the bins miss the optimal split point corresponding to the maximum of the objective function.}
\label{fig:mismatch}
\end{figure}

\section{Mathematical framework for choosing the subset}
\label{sec:math}
We now provide a precise mathematical definition of the subset and the role it plays. Let $\X = x_1, x_2, x_3 \hdots x_n$ be set of data points over which we would like to learn a function, using a decision tree $T$. Without loss of generality, we assume $x_1 \leq x_2 \leq x_3 \hdots \leq x_n$. We define a ranking on $\X$ over an objective function $f: \X \rightarrow \R$ where $f$ is a tree score function that provides a greedy measure of the quality of the split and is usually defined on the partition of the points. The objective is to find $x_{max}$ such that $f(\{x_1 \hdots x_{max}\},\{x_{{max} + 1} \hdots x_n\})$ is the maximum amongst all the ordered partitions defined on $\X$. Conventional algorithms to build decision trees iterate through $i=1 \hdots n$ to find the maximum value of $f(\xl_i,\xr_{i})$ where $\xl_i = \{x_1 \hdots x_i\}$ and $\xr_i = \{x_{i+1} \hdots x_n\}$. For ease of notation, we express $f(\xl_i,\xr_{i})$ as $f(x_i)$. The linear scan algorithm is inefficient when $n$ is large. An efficient alternative to this algorithm is to choose a subset of $\X$ to evaluate $f$. Let us denote the subset chosen by the approximation algorithm to be $\SB$ such that $\SB \subset \X$ and the element in $\X$ that maximises $f$ is $x^*$. We use $\SB$ as a set of candidate splits. Evaluating the splits over a smaller set $\SB$ compared to $\X$ would be more efficient. The efficacy of the approximation algorithm is measured by the ability of the algorithm to include $x^*$ in $\SB$. Less drastic measure would be to measure the {\em rank error} of $\SB$ with respect to $\X$ by finding the element in $\SB$ that has the maximum value of $f$, say, $x^*_{\SB}$. The rank of $x^*_{\SB}$ in $\X$ would give the error induced by the approximation algorithm. Let us denote the rank error as $R_f(\SB,\X)$ which takes a value between $0$ and $n-1$. Hence, if $\SB$ contains $x^*$, $R = 0$, if $\SB$ does not contain $x^*$ but contains the next best, then $R = 1$ and so on. We can formally define $R$ as -
\begin{equation}
R = \frac{\text{rank of the highest ranked element in S}}{|S|-1}
\end{equation}
where $|S|$ is the cardinality of $\SB$.

There are usually two approaches employed by various algorithms when constructing $\SB$ from $\X$ as described in Section \ref{sec:related}. In the \data approach $\SB$ is chosen to be {\em faithful} to $\X$ through some appropriate definition of faithfulness. In the \objective approach, $\SB$ is constructed to be faithful to $f$ wherein the attempt is to include the best possible data point in $\SB$. It must be noted that $R$ is always defined over a ranking of the data points over $f$. 
We can now quantify the expected rank error if we were to construct $\SB$ using uniform random sampling and use this as the baseline to compare more sophisticated ways of choosing $\SB$.
\subsection{Random selection of bins}
\label{sec:rand}
We start by examining a simple approximation scheme where $\SB$ is constructed by selecting elements uniformly random from $\X$. We begin with a theorem that provides a measure of the expected rank error of a set $\SB$ chosen uniformly random from $\X$.

\begin{theorem}
\label{thm1}
The expected rank error over all possible subsets of $\X$ of cardinality $k$ is $\frac{1}{k+1}$, when the subsets are chosen uniformly at random.
\end{theorem}
\begin{proof}
The probability that the optimal element $x^*$ is included in $\SB$ is denoted by $P(x^* \in \SB)$ and is given by -
\begin{equation}
P(x^* \in \SB) = \frac{{n-1 \choose k-1}}{{n \choose k}} 
\end{equation}
Similarly, the probability that $x^*$ is not chosen in $\SB$ but the next best $x$ denoted by $x^*_2$ is chosen is given by -
\begin{equation}
P(x^*_2 \in \SB, x^* \notin \SB) = \frac{{n-2 \choose k-1}}{{n \choose k}} 
\end{equation}
Hence the expected rank error is given by -
\begin{eqnarray}
\mathop{\mathbb{E}}[R] &=&  \sum_{i = 1 \hdots k}{{P(x^*_{i} \in \SB, \{x^*_1, \hdots x^*_{i-1}\} \notin \SB)(i-1)}} \\
&=& \sum_{i = 1 \hdots n-k+1}{\frac{{n-i \choose k-1}}{{n \choose k}}(i-1)}
\label{eqn:Et}
\end{eqnarray}
We can simplify this to $\mathop{\mathbb{E}}[R] = \frac{n-k}{k+1}$. (proof in Appendix Section 6.2)
\end{proof}
Hence from the theorem, we find that if we were to choose the subset of candidate split points from the data in a random fashion, then the expected error of the rank of the best element in the chosen subset would be inversely related to the size of the subset chosen. For ease of analysis, we normalise the expected error by dividing the expected error by the worst possible error $(n-k)$ to get normalised error $E$ as -
\begin{equation}
E = \frac{1}{n-k}\mathop{\mathbb{E}}[R] = \frac{1}{k+1}
\end{equation}

\subsection{Deterministic selection of bins}
There are other possible ways of choosing $\SB$ using some of the popular methods of quantile computations. These algorithms provide a good approximation of the entire data $\X$ efficiently. In this discussion, however, our focus is on building decision trees with these approximated sketches and belong to the category of {\em data faithful} approaches. The argument we make here is that it is not enough to have an approximation of the big data to ensure a good decision tree being built. Instead, we need an approximation of the data that would be faithful to $f$, which would ensure that the candidate split points obtained from the data approximation has a good probability of including the points that have the best split candidates. XGBoost uses a modification of GK summary\cite{gk}, and CATBoost uses a fixed number of bins with fixed ranges to quantize the data. 

We first consider a simplified variant of the approximation strategy used in the popular algorithm of XGBoost. XGBoost originally uses a weighted version of the GK summary to construct the approximate data. We consider the GK summary algorithm where there are no weights associated with the data points. Bins are then constructed in such a way that the rank of a point can be approximated such that the error in the rank is less than  $\epsilon$. Hence, we expect to have as many bins as $1/\epsilon$, which means that $\epsilon \approx 1/k$. In \cite{xgboost}, the authors prove that such an allocation of bins and consequent selection of $\SB$ results in a normalized error of less than $\epsilon$ for a query made for a quantile. In its simplest implementation, the quantile sketch is used to query for rank of the data in $\X$. However, in our case, we need to find the maximum element in $f$, which we cannot find using the quantile sketch unless $f$ is monotonic over $\X$. Hence, the algorithm is equivalent to a random selection of points since there is no correlation between the bins constructed out of $\X$ and $f$. 
In \fig{fig:randvsgk}, we compare the random selection of $\SB$ with the deterministic selection of bins. In the experiment, we sample $\X$ uniformly random and select $\SB$ first using random selection as described in \sect{sec:rand}. We approximate the data set by binning it into $k$ buckets, and we measure the expected rank error $E$ by averaging over many such runs. The expected error is measured over different bin sizes and is plotted in \fig{fig:randvsgk}. The plots in the figure support our claim that the random selection of $\SB$ and the deterministic selection of $\SB$  using the algorithm in \cite{xgboost} are not significantly different in terms of their rank errors. In XGBoost, the quantile sketch is modified to have weighted quantiles. The weights are designed to weigh points based on their contribution to the error, yet it still does not correlate with the objective used to decide on the split, hence it still would be \data than \objective.
We have illustrated the effect of \data on GK summary used in XGBoost, but this analysis applies to all other algorithms belonging to the \data class.
\begin{figure}
\includegraphics[width=0.35\textwidth]{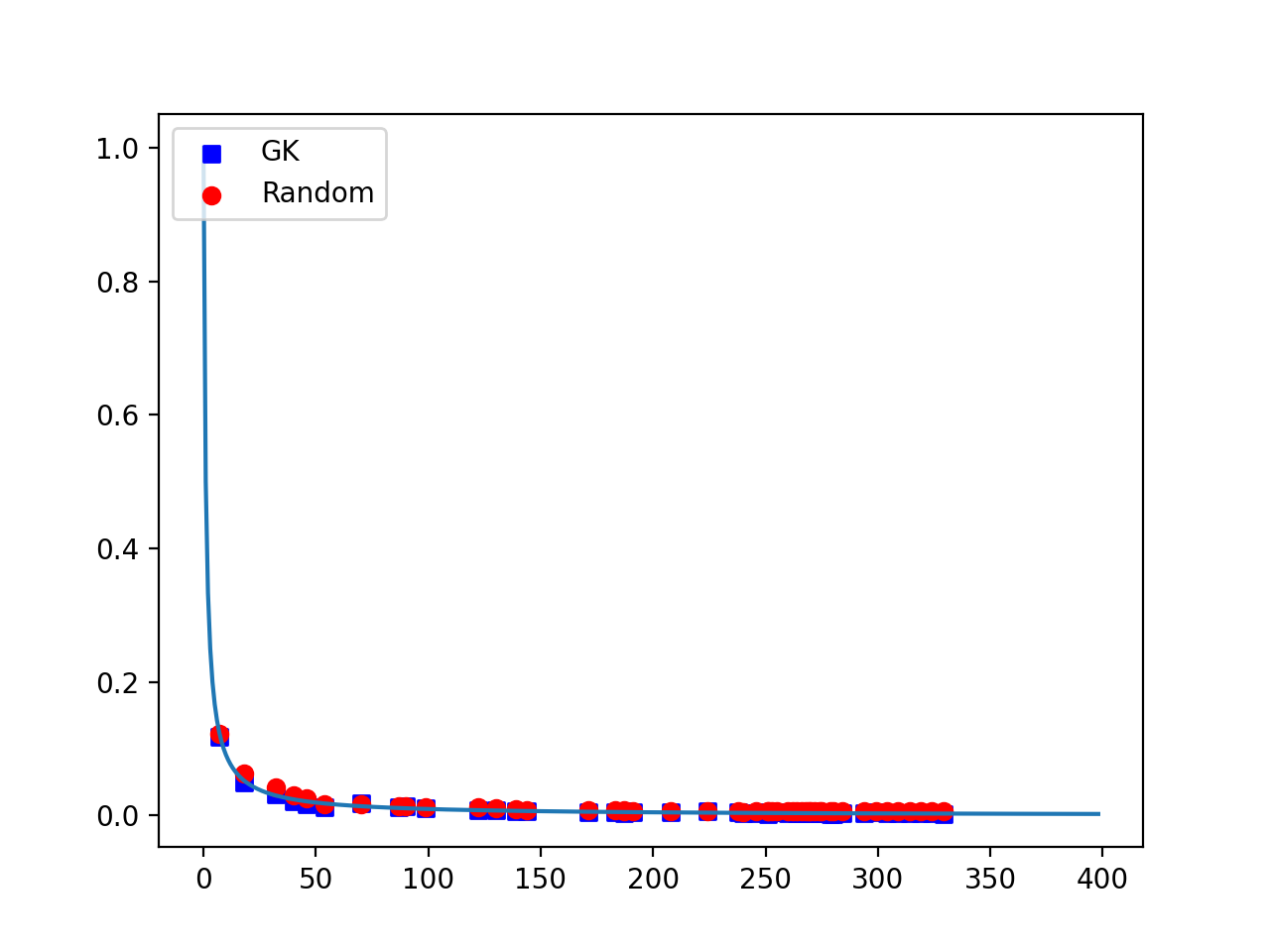}
\caption{Comparison of rank error for random selection vs GK summary. The points denote the mean error for the rank of selected point. The bold line corresponds to $1/(k+1)$}
\label{fig:randvsgk}
\end{figure}

\section{Evaluations}
\label{sec:eval}
In our experiments, we compare the accuracy achieved by XGBoost using a simple random sampling of points versus XGBoost using weighted quantile building algorithm. Random sampling is incorporated in XGBoost by performing local sampling during data reading and performing global sampling during split point proposal. Since building the subset primarily affects the split point chosen and can result in a non-greedy decision tree being built, so, as a first step we compare the quality of the decision tree (1 tree XGBoost) built between the two methods. The accuracies (for classification task) / mean absolute percentage error (MAPE for time series regression task) are compared on datasets\cite{noniiddata}\cite{uci} described in Table \ref{tab:desc} and the results are presented in Table \ref{tab:res}. For both of the methods, we used default settings for the parameters and distributed the computation over $30$ workers. The number of bins used for the experiments is set based on the training dataset size. The evaluations using different bins have been reported as an average over $5$ runs. Using a similar setup we also built an ensemble of $20$ decision trees ($50$ for regression) using XGBoost.

From the results on different datasets shown in Table 2, it is clear that by using a simple random selection of candidate split points for building decision trees, we can achieve the same levels of accuracy as using sophisticated quantile building algorithm thus validating our claim. We can see in Table 2 that the time taken by random sampling is substantially lesser than using quantile approximation as would be expected. We also observe that random sampling is able to handle non-iid data and is performing equivalent to quantile algorithm. We also measured the variance of accuracies/errors across runs and observed that the variance is $< 0.001$, so we can consider random sampling to give stable boosting results. By using a simpler method, we are able to improve the computational efficiency of the algorithm which is important while processing big data. 

\begin{table}
\footnotesize
\begin{tabular}{|c|c|c|c|}
\hline
Dataset & \#features &  \#train &  \#test  \\
\hline
Wiretap\cite{Mirsky2018KitsuneAE} (WT - Class) & 115 & 200000 & 50000\\
\hline
Mirai\cite{Mirsky2018KitsuneAE} (MI - Class) & 115 & 563137 & 100000\\
\hline
SUSY\cite{higgs} (SU - Class) & 18 & 4500000 & 500000\\
\hline 
Hepmass\cite{hepmass} (HM - Class) & 28 & 7000000 & 3500000\\ 
\hline
Higgs\cite{higgs} (HI - Class) & 28 & 10500000 & 500000\\
\hline
PJM East\cite{noniiddata} (PJM - Reg) & 10 & 110000 & 35366\\ 
\hline
Dominion Vir \cite{noniiddata} (DOM - Reg) & 10 & 84750 & 31439\\
\hline
\end{tabular}
\caption{Datasets (Class: classification, Reg: regression)}
\label{tab:desc}
\end{table}

\begin{table}
\footnotesize
\begin{tabular}{|c|c|c|c|c|c|c|c|}
\hline
Data & Bins & DT(S) &  DT(Q) & XGB(S) &  XGB(Q) & T(S) & T(Q) \\
\hline
WT & 10 & 0.994 & 0.995 & 0.997 & 0.995   & 11.5 & 21.2  \\ 
\hline
 & 20 & 0.989 & 0.995 & 0.997 & 0.996 & 12.3 & 23.7 \\ 
\hline
 & 50 & 0.996 & 0.996 & 0.997 & 0.997  & 17.5 & 25.7 \\ 
\hline
 & 100 & 0.996 & 0.997 & 0.998 & 0.998  & 24.7 & 33.2 \\ 
\hline
MI & 10 & 0.972 & 0.986 & 0.992 & 0.992  & 11.2 & 74.3 \\ 
\hline
 & 20 & 0.982 & 0.985 & 0.993  & 0.992  & 13.0 & 76.0 \\ 
\hline
 & 50 & 0.985 & 0.992 & 0.993 & 0.993  & 17.8 & 76.8 \\ 
\hline
 & 100 & 0.991 & 0.992 & 0.993 & 0.993 &  26.3 & 78.8 \\ 
\hline
SU & 10 & 0.77 & 0.772 & 0.785 & 0.785  & 36.7 & 106.2 \\ 
\hline
 & 100 & 0.773 & 0.772 & 0.785 & 0.784  & 30.7 & 102.0 \\ 
\hline
 & 500 & 0.772 & 0.772 & 0.785 & 0.785  & 43.0 & 104.2 \\ 
\hline
 & 1000 & 0.773 & 0.772 & 0.784 & 0.784  & 60.2 & 111.3 \\ 
\hline
HM & 10 & 0.837 & 0.837 & 0.855 & 0.854  & 43.7 & 218.8 \\ 
\hline
 & 100 & 0.837 & 0.836 & 0.854 & 0.854  & 41.2 & 209.2 \\ 
\hline
 & 500 & 0.837 & 0.837 & 0.855 & 0.854  & 56.0 & 228.5 \\ 
\hline
 & 1000 & 0.836 & 0.837 & 0.855 & 0.855  & 81.5 & 225.2 \\ 
\hline
HI & 10 & 0.675 & 0.677 & 0.707 & 0.707  & 69.8 & 401.8 \\ 
\hline
 & 100 & 0.678 & 0.679 & 0.707 & 0.707  & 70.2 & 403.7 \\ 
\hline
 & 500 & 0.679 & 0.679 & 0.706 & 0.708  & 83.7 & 401.8 \\ 
\hline
 & 1000 & 0.679 & 0.679 & 0.708 & 0.708  & 106.8 & 410.5 \\ 
\hline
PJM & 10 & 89.422 & 89.419 & 10.838 & 10.906  & 1.0  & 6.3  \\ 
\hline
 & 20 & 89.424 & 89.424 & 10.835 & 10.824 & 1.4& 7.1 \\ 
\hline
 & 50 & 89.423 & 89.424 & 10.839 & 10.811 & 1.6 & 7.5 \\ 
\hline
 & 100 & 89.421 & 89.422 & 10.832 & 10.818 & 1.3 & 5.7 \\ 
\hline
DOM & 10 & 90.063 & 89.991 & 14.074 & 14.087 & 0.9 & 4.9 \\ 
\hline
 & 20 & 90.007 & 89.901 & 14.100 & 14.008 & 1.4 & 4.4 \\ 
\hline
 & 50 & 89.908 & 89.872 & 14.046 & 13.999 & 1.3 & 5.0 \\ 
\hline
 & 100 & 89.985 & 89.863 & 14.056 & 13.960 & 1.5 & 4.8 \\ 
\hline
\end{tabular}
\caption{Table presents accuracy/error of Decision Tree(DT), XGBoost(XGB) and time taken (T in ms) for sampling (S) vs deterministic quantile building (Q)}
\label{tab:res}
\end{table}

\section{Conclusion and future work}
In this paper, we critically examine the methods used to build decision trees using a quantile sketch algorithm to approximate the data. We proved that just having \data methods that approximate the data distribution without linking it with the objective function $f$ does not help in providing any improvement over a random selection of candidate split points. Using this strong result, we argue and empirically prove that using a random selection of candidate split points provides the same level of accuracy as a sophisticated algorithm for finding split points. A random selection of points would be more time-efficient compared to other methods providing gains in running boosting algorithms. In terms of the complexity of the software, implementing a distributed decision tree based on a random selection is much simpler than sophisticated quantile building algorithms.

In future, we would like to develop \objective algorithms that are also as computationally efficient as the \data algorithms. We would also like to modify other GBDT methods like CATBoost and LightGBM to use random sampling.

\bibliographystyle{ACM-Reference-Format}
\bibliography{biblio}

\section{Appendix}
\subsection{Incorporating Random Sampling in XGBoost}
\label{sec:sampl}
This section will describe how we replaced the quantile algorithm with the random sampling algorithm in XGBoost. During the first step of data reading, each node in a distributed setting randomly samples from the local data it reads. If $b$ bins are given as input for each feature, the sample size read per feature by each node will be $b$. After reading the data, the boosting starts. We need to propose candidate split points for each iteration from which the best split point will be found. This proposal can be done by weighted quantile sketch or using random sampling. Now it is required that all nodes have the same set of candidate split points. So, after each node proposes its local candidate split points (in case of random sampling local proposal done during data reading), an AllReduce\cite{Chen2015RABITA} operation is called (AllReduce is reducing and then broadcasting). In case of random sampling, all reduce will combine the samples and then sample from the set again to ensure that the sample size for a feature is at most $b$. After broadcasting the samples, the further steps are shared with what is done in XGBoost. The pseudo code of the distributed XGBoost algorithm is given in Algorithm \ref{alg:xgboost}.

\begin{algorithm}
\caption{XGBoost distributed algorithm}
\begin{algorithmic} 
\STATE $dmatrix \leftarrow$ Local data loaded and local sampling done by each node in distributed mode 
\STATE // Done by each node
\FOR{$i = 0; i < num\_iterations; i++$} 
\STATE $tree \leftarrow $ Init tree for boosting round
\STATE  $candidate\_split\_points \leftarrow$  \textbf{quantiles /random sampling}
\FOR{$d = 0; d < max\_depth; d++$} 
\STATE $hist \leftarrow $ gradient statistics computation
\STATE $splits \leftarrow $ find split for this depth
\ENDFOR
\ENDFOR
\end{algorithmic}
\label{alg:xgboost}
\end{algorithm}

\subsection{Proof for theorem \ref{thm1}}
We start with \eref{eqn:Et},
\begin{eqnarray}
E &=& \frac{1}{{n \choose k}}\sum_{r = 0}^{n-1}{(n - r - 1){r \choose k-1}}
\end{eqnarray}
where $n - i = r$.
Now, ${a \choose b} = 0$ for $a < b$\\
\\
So, ${r \choose k-1} = 0$ for $r < k - 1$\\
\begin{eqnarray}
E &=& \frac{1}{{n \choose k}}\sum_{r=k-1}^{n-1}{(n - r - 1){r \choose k-1}}\\
&=& \frac{1}{{n \choose k}}\left[\sum_{r=k-1}^{n-1}{(n - 1){r \choose k-1}} - \sum_{r=k-1}^{n-1}{r{r \choose k-1}}\right]\\
&=& \frac{1}{{n \choose k}}\left[(n - 1)\sum_{r=k-1}^{n-1}{{r \choose k-1}} - \sum_{r=k-1}^{n-1}{r{r \choose k-1}}\right]
\label{eqn:E}
\end{eqnarray}
There are two parts of this equation which will be simplified separately.
The first part is: 
\begin{equation}
(n-1)\sum_{r=k-1}^{n-1}{{r \choose k-1}}
\label{eqn:firstpart}
\end{equation}
For simplifying the first part, we would be using the following identity repeatedly:
\begin{equation}
{x \choose y} + {x \choose y+1} = {x+1 \choose y+1}
\label{eqn:iden}
\end{equation}
We expand \eref{eqn:firstpart},
\begin{eqnarray}
\sum_{r=k-1}^{n-1}{{r \choose k-1}} &=& {k-1 \choose k-1} + {k \choose k-1} \hdots {n-1 \choose k-1} \\
&=& {k \choose k} + {k \choose k-1} \hdots {n-1 \choose k-1} 
\label{eqn:kchoosek}
\end{eqnarray}
where we have used the fact that ${k-1 \choose k-1} = {k \choose k}$. Applying \eref{eqn:iden} to the first two terms of \eref{eqn:kchoosek} we get,
\begin{equation}
\sum_{r=k-1}^{n-1}{{r \choose k-1}} = {k+1 \choose k} + {k+1 \choose k-1} + {k+2 \choose k-1} \hdots {n-1 \choose k-1}\\
\end{equation}
Repeating the process over the terms we get,
\begin{equation}
\sum_{r=k-1}^{n-1}{{r \choose k-1}} = {n \choose k}
\label{eqn:t1}
\end{equation}
We can now simplify the second part of the equation,
\begin{equation}
\sum_{r=k-1}^{n-1}{r{r \choose k-1}}
\end{equation}
For this too, we will start by expanding it into a series:
\begin{equation}
\sum_{r=k-1}^{n-1}{r{r \choose k-1}} = (k-1){k-1 \choose k-1} + k{k \choose k-1} \hdots (n-1){n-1 \choose k-1}
\end{equation}
Let $x = n-k$, and using the property ${a \choose b} = {a \choose a-b}$\\
\begin{equation}
\sum_{r=k-1}^{n-1}{r{r \choose k-1}} = (k-1){k-1 \choose 0} + k{k \choose 1} + \hdots + (k-1+x){k-1+x \choose x}
\end{equation}
which can be regrouped as,
\begin{multline}
= (k-1)\left[ {(k-1) \choose 0} + {k \choose 1} + \hdots + {k-1+x \choose x} \right] + \\
 \left[ {k \choose 1} + {k+1 \choose 2} + \hdots + {k-1+x \choose x} \right]\\
+ \left[ {k+1 \choose 2} + {k+2 \choose 3} + \hdots + {k-1+x \choose x} \right] + \hdots + \left[ {k-1+x \choose x} \right]
\label{eqn:mult}
\end{multline}
Now,
\begin{equation}
{k-1 \choose 0} + {k \choose 1} + \hdots + {k-1+x \choose x} = {k \choose 0} + {k \choose 1} + \hdots + {k-1+x \choose x}
\end{equation}
Simplifying using \eref{eqn:iden} recursively,
\begin{equation}
{k \choose 0} + {k \choose 1} + \hdots + {k-1+x \choose x} = {k+x \choose x} = Z
\label{eqn:defz}
\end{equation}
Simplifying \eref{eqn:mult} in terms of Z, we get,
\begin{multline*}
= (k-1)Z + \left[ Z - {k-1 \choose 0} \right] + \left[ Z - \left( {k-1 \choose 0} + {k \choose 1} \right) \right]\\
+ \hdots + \left[ Z - \left( {k-1 \choose 0} + {k \choose 1} + \hdots + {k+x-2 \choose x-1}  \right)\right]\\
\end{multline*}
\begin{multline*}
= (k-1)Z + xZ - \left[ {k-1 \choose 0} +\hdots + \left( {k-1 \choose 0} + \hdots+ {2*(k-1) \choose k-1} \right) \right]\\
\end{multline*}
Simplifying using \eref{eqn:iden} recursively,
\begin{eqnarray}
&=& (k-1+x)Z - \left[ {k \choose 0} + \hdots + {k-1+x \choose x-1} \right]\\
&=& (k-1+x)Z - \left[ {k+1 \choose 0}  + \hdots + {(k+1)+(x-2) \choose x-1}  \right]
\end{eqnarray}
Using \eref{eqn:iden} to simplify,
\begin{eqnarray}
&=& (k-1+x)Z - {(k+1) + (x-2) + 1 \choose x-1}\\
&=& (k-1+x)Z - {k+x \choose x-1}
\end{eqnarray}
Replace $x=n-k$ and Z from \eref{eqn:defz},
\begin{eqnarray}
&=& (n-1){n \choose n-k} - {n \choose n-k-1}\\
\sum_{r=k-1}^{n-1}{r{r \choose k-1}} &=& (n-1){n \choose n-k} - {n \choose n-k-1}
\label{eqn:second}
\end{eqnarray}
Now simplifying \eref{eqn:E} using \eref{eqn:t1} and \eref{eqn:second},
\begin{eqnarray}
E &=& \frac{1}{{n \choose k}}\left[ (n-1){n \choose k} - \left( (n-1){n \choose k} - {n \choose k+1} \right) \right]\\
&=& \frac{{n \choose k+1}}{{n \choose k}}\\
E &=& \frac{n-k}{k+1}
\end{eqnarray}

\end{document}